\documentclass[12pt]{article}
\usepackage{amsmath}
\usepackage{graphicx,psfrag,epsf}
\usepackage{enumerate}
\usepackage{natbib}
\usepackage{url} 

\usepackage{amsthm}
\usepackage{amsmath}
\usepackage{amssymb}
\usepackage{mathtools}

\usepackage{xcolor}

\usepackage{rotating}
\usepackage{pdflscape}
\usepackage{graphicx}

\usepackage{comment}

\usepackage{bibentry}
\usepackage{appendix}

\usepackage{enumitem}

\newtheorem{theorem}{Theorem}[section]

\newtheorem{assumption}{Assumption}[section]

\usepackage{times}
\usepackage{bm}
\usepackage{natbib}
\usepackage{setspace}
\def\l{\left}
\def\r{\right}
\def\X{\textbf{X}}

\def\E{\mathbb{E}}

\def\1{\textbf{1}}

\def\x{\textbf{x}}

\def\argmax{\arg\max}
\DeclarePairedDelimiter{\ceil}{\lceil}{\rceil}

\begin{document}

\def\spacingset#1{\renewcommand{\baselinestretch}%
{#1}\small\normalsize} \spacingset{1}


\bigskip
\bigskip
\bigskip
\begin{center}
{\LARGE Technical Background for ``A Precision Medicine Approach to Develop Optimal Exercise and Weight Loss Treatments for Overweight and Obese Adults with Knee Osteoarthritis''}
\end{center}
\bigskip
\begin{center}{X.Jiang${}^1$, A.E.Nelson${}^2$, R.J.Cleveland${}^2$, D.P.Beaver${}^3$, T.A.Schwartz${}^1$, L.Arbeeva${}^2$, \\  C.Alvarez${}^2$, L.F.Callahan${}^2$, S.Messier${}^4$, R.Loeser${}^2$, M.R.Kosorok${}^1$}\end{center}

\begin{center}
\emph{${}^1$Department of Biostatistics,  University of North Carolina, Chapel Hill, NC,\\${}^2$Division of Rheumatology, Allergy and Immunology and the Thurston Arthritis Research Center,	University of North Carolina, Chapel Hill, NC,\\${}^3$Department of Biostatistics and Data Science, Wake Forest School of Medicine, Winston-Salem, NC,\\${}^4$Department of Health and Exercise Science, Wake Forest University, Winston-Salem, NC}
\end{center}
\medskip

\bigskip
\begin{abstract}
This technical report provides additional statistical background for the methodology developed in the clinical analysis of knee osteoarthritis in \cite{jiang2020}. \cite{jiang2020} proposed a pipeline to learn optimal treatment rules with precision medicine models and compared them with zero-order models with a Z-test. The model performance was based on value functions, a scalar that predicts the future reward of each decision rule. The jackknife (i.e., leave-one-out cross validation) method was applied to estimate the value function and its variance of several outcomes in IDEA, a randomized clinical trial studying overweight and obese participants with knee osteoarthritis. In this report, we expand the discussion and justification with additional statistical background. We elaborate more on the background of precision medicine, the derivation of the jackknife estimator of value function and its estimated variance, the consistency property of jackknife estimator, as well as additional simulation results that reflect more of the performance of jackknife estimators. We recommend reading \cite{jiang2020} for clinical application and interpretation of the optimal ITR of knee osteoarthritis as well as the overall understanding of the pipeline and recommend using this article to understand the underlying statistical derivation and methodology.
\end{abstract}

\noindent%
\vfill

\newpage
\spacingset{1.45} 
\section{Introduction}
\label{sec:intro}

Precision medicine (a.k.a personalized medicine) is a data-driven paradigm that aims to improve clinical outcomes by tailoring treatments to patients as individualized as possible. For example, cancer patients usually receive the same treatment as others who have the same type and stage of cancer, but different people could respond to the same treatment differently. Scientists now start to understand that genetics can cause the tumors to develop and spread, providing an explanation to individual differences in cancer treatment. With precision medicine, such information about genetic changes in the tumor helps clinicians choose a treatment plan that works best for each individual in terms of their well-being and outcome of interest, and it also offers the opportunity to activate other areas of research involving drug design that target tumor changes more effectively. Although not a new concept, precision medicine has been propelled rapidly by recent advances in biomedical sciences and technology and it is both reproducible and generalizable. 

A rigorous and reproducible decision-making process requires four well-defined components: patient information, intervention plans, decision time points, and the decision rules. Patient information is the available data that include both prognostic covariates and clinical outcomes; the intervention plan refers to a detailed guide that informs the type/dose/duration of the treatment tailored to individual patients; decision time points are when and how often the decisions are made; decision rules are functions that map patient information to interventions at each decision point \citep{collins2014optimization}. Precision medicine is formalized through such decision rules that recommend treatments by leveraging patient heterogeneity and are often called individualized treatment rules (ITRs)  \citep{qian2011performance}. ITR is essentially a mathematical emulation of how clinicians make treatment decisions in practice. 

To provide a concrete example, we consider a study of weight and glycemic control among Type-1 diabetic (T1D) patients, whose glucose  and physical activity were monitored continuously with trackers and diet and insulin recorded whenever taken \citep{kahkoska2019}. This is a special decision-making process called dynamic treatment regime that is commonly used in the treatment of chronic and relapsing disorders \citep{collins2014optimization}. Patient information for T1D patients includes glucose level, step counts, diet, and insulin intake, as well as medical history and lab results from regular clinical visits. The treatment plan at each time point is whether or not each patient needs to take an action, e.g. inject insulin, work out to increase activities, or consume food; all plans are personalized. The decision rule decides whether the individual needs to take an action at each time point. If yes, a recommendation of what action to take will be given. If no, nothing needs to be done. A the next time point, responses for the previous time point are received and incorporated as part of patient information to help make the next decision. This decision-making process is dynamic not only because there are multiple time points (in fact an infinite amount if we think of this an ongoing process) but also the treatment regime is adapted to individual characteristics, as opposed to a fixed, ``one-size-fits-all'' rule.

\section{Existing Work}\label{sec:existing}
We consider a single time point setting with $n$ independent and identically distributed copies of $\{\X_i, A_i, Y_i\}$ for $i = 1, \ldots, n$ patients. Patient covariates are $\X_i \in \mathcal{X} \subseteq \mathbb{R}^p$, a $p \times 1$ vector. Treatment is a scalar $A_i \in \mathcal{A}$, and the outcome is also a scalar $Y \in \mathbb{R}$. Throughout this paper, we assume that higher values of $Y$ represent more favorable clinical outcomes unless otherwise specified. Existing research has been studying precision medicine based on data from both i) randomized clinical trials where treatment assignment is independent of patient covariates hence $\X \perp A$ and known propensity scores $P(A|\X)$, and ii) observational studies and not-randomized trials where $\X \not\perp A$ and propensity scores need to be estimated by logistic or multinomial logistic regression depending on the number of treatments. The decision rule is defined as a function $d: \mathcal{X} \to \mathcal{A}$ where a treatment assignment $A_i$ is obtained and $A_i = d(\X_i)$ given the $i$-th individual's clinical covariates $\X_i$. According to \cite{qian2011performance}, the single-stage decision problem finds the optimal ITR $d^{opt}$ in a class of all possible decision rules $\mathcal{D}$ that maximizes the ``Value'' or the expected reward of a potential outcome under a decision rule when applied to future patients. Mathematically, this is expressed as 
\begin{equation} \label{eqn: d_opt_Vd}
d^{opt} = \argmax_{d \in \mathcal{D}} \ V(d),
\end{equation}
where 
\begin{equation} \label{eqn: Vd}
V(d) = E^d [Y] = E\l[ \frac{Y 1\{A = d(\X)\}}{P(A|\X)}\r]
\end{equation}
is derived by Radon-Nikodym derivatives.  The optimal ITR also satisfies $$d^{opt}(\x) \in \argmax_{d \in \mathcal{D}} \ Q(\x,a) \text{ a.s. for all } \x \in \mathcal{X}$$ where $Q(\x,a) = E[Y| \X = \x, A = a]$ is the ``Quality'' of treatment $a$ applied at patient observation $\x$. The relationship between the value and the quality is $V(d) = E[Q(\X, d(\X))]$. The goal of precision medicine is to estimate $d^{opt}$ using $n$ triplets of observed data $\hat{d}_n$.  

Approaches to estimating the optimal treatments can be summarized into two major types: regression-based and classification-based. Neither of them requires pre-specified values or strong assumptions which is one of the reasons that makes precision medicine data-driven. Regression-based approaches (the indirect methods) posit a traditional statistical regression model of $Q$, the mean outcome conditional on patient data and treatment assignment as described above (\cite{murphy2003optimal}, \cite{robins2004optimal}, \cite{moodie2009estimating}, \cite{moodie2014q}, \cite{taylor2015reader}). The optimal decision rule picks, among all potential treatments $\mathcal{D}$, one treatment that yields the highest estimated mean outcome based on the model conditioning on patient data, $\hat{Q}$. For example, \cite{qian2011performance} achieved this with a two-step procedure with the $\ell_1$-penalized least squares model. \cite{zhang2012robust} went on an alternative path with the counterfactuals. Consider all possible decision rules indexed by $\eta$ as in $d_\eta(\X) = d(\X, \eta)$ and denote the potential outcome following decision rule $d$ as $Y^d$. The optimal ITR in \cite{zhang2012robust} was obtained by estimating $\eta^{opt} = \argmax_\eta E[Y^{d_\eta}]$ with augmented inverse probability weighting estimator (AIPWE) (\cite{robins1994estimation}, \cite{rotnitzky1998semiparametric}) and defining $d^{opt}_\eta = d(\X, \eta^{opt})$.  This doubly robust approach exploits the outcome regression model to gain estimation precision while being protected by doubly robustness in case the regression model or propensity score model are misspecified. 

An alternative to avoid the potential misspecification of regression models is classification-based approaches (the direct methods) where the optimal ITR is estimated by the optimal classifier that minimizes the expected weighted misclassification error. A leading example of such approach is \cite{zhao2012estimating} where the proposed Outcome Weighted Learning (OWL) model avoids regression and model assumption by converting the estimation problem to a weighted classification problem and applying weighted support vector machine (SVM) techniques. We used OWL and its extensions as part of the precision medicine candidate models in the clinical analysis and simulations described in \cite{jiang2020} and Section \ref{sec:sims}.

\section{Estimating Value Function with the Jackknife Method}\label{sec:pmoa_jackknife}
The value function is often used to measure the performance of an ITR so it is important to have an accurate estimation of its bias and standard error. The basic form of a value function estimator is
\begin{equation}\label{eqn:pmoa_vhat}
\widehat{V}(d) = \frac{\sum_{i=1}^n Y_i 1\{A_i = d(\X_i)\}/ \hat{P}(A_i | \X_i)}{\sum_{i=1}^n 1\{A_i = d(\X_i)\}/ \hat{P}(A_i | \X_i)} 
\end{equation}
which can be deemed as a weighted combination of individual outcomes. Because $d$ is unknown and needs to be estimated, a common method uses cross validation (CV). Assume a $K$-fold CV repeated $M$ times. Let $j = 1, \ldots, KM$ denote all CV folds regardless of repetition and $i=1,\ldots,n_j$ denote the $i$th observations in the $j$-th overall CV fold. The CV estimator of value function is \begin{equation}\label{eqn:stratcv_vhat}
\widehat{V}^{cv}(\hat{d}_n) = \frac{\sum_{j=1}^{MK}\sum_{i=1}^{n_j}  Y_{ji} \frac{1\{A_{ji} = \hat{d}_{n}^{(-j)} (\X_{ji}) \}}{\hat{P}(A_{ji}|\X_{ji})} }{\sum_{j=1}^{MK} \sum_{i=1}^{n_j}  \frac{1\{A_{ji} = \hat{d}_{n}^{(-j)} (\X_{ji}) \}}{\hat{P}(A_{ji}|\X_{ji})}}
\end{equation}
where $\hat{d}_{n}^{(-j)}$ is the estimated decision rule leaving out the $j$th CV fold of a training dataset of size $n$. $\hat{P}(A_{ji}|\X_{ji})$ is the estimated propensity score of the test fold $ji$ (constant if treatments are randomized). $k$-fold CV requires many repetitions because a one-time split could generate random biases, and we are unable to utilize as much training data as possible because we always leave one fold out as test set. 

Given these disadvantages, we proposed to use the jackknife method 
, a.k.a. leave-one out cross validation (LOOCV), to estimate the bias and standard error of value function \citep{jiang2020}. The jackknife is a special case of CV, where data are divided such that one subject is a fold and $n-1$ folds are trained and then tested on the remaining one fold, repeatedly until all folds have been used as a test set once. According to \cite{jiang2020}, jackknife makes very few assumptions about the data distribution, only basic constraints such as i.i.d. independently and identically distributed. Compared with a $k$-fold CV, it uses as much data to train as possible generating approximately unbiased prediction error and avoids repetitions. Both the jackknife and $k$-fold CV were used to estimate value functions and measure model performance. For clinical results in knee osteoarthritis, see the results section in \cite{jiang2020} and ``Stratified Cross Validation'' subsections in its supplemental material.

\subsection{The Jackknife Estimator} \label{sec:definition}
The jackknife value estimator proposed in \cite{jiang2020} has the following form \begin{equation}\label{eqn:jackknife_vhat}
\widehat{V}^{jk} \l( \hat{d}_n\r) = \frac{\sum_{i=1}^n  Y_i \frac{1\{A_i = \widehat{d}_n^{(-i)}(\X_i)\}}{\hat{P}(A_i |\X_i)}}{\sum_{i=1}^n \frac{1\{A_i = \widehat{d}_n^{(-i)}(\X_i)\}}{\hat{P}(A_i |\X_i)}} 
\end{equation} 
where $\hat{d}_n^{(-i)}$ represents the decision rule estimated from a training set of size $n$ minus the $i$-th observation, and similar to the CV version, $\hat{P}(A_i|\X_i)$ is the estimated propensity score of the test set $i$ (constant if treatments are randomized). Eq \ref{eqn:jackknife_vhat} is a special case of Eq \ref{eqn:stratcv_vhat} with $M = 1$ and $K = n$. The estimated variance of the jackknife value estimator is 
\begin{equation}\label{eqn:jackknife_vhat_var}
\widehat{\text{Var}}\l[ \hat{V}^{jk} \l( \hat{d}_n \r)\r] = \frac{1}{n(n-1)} \sum_{i=1}^n R_i^2
\end{equation}
where $R_i^{jk} = \frac{1}{\bar{W}_n} U_i - \frac{\bar{U}_n}{\bar{W}_n^2} W_i$ is the bias-corrected form of value function inspired by the influence function, with $U_i = \frac{Y_i 1\{A_i = \hat{d}_n^{(-i)} (\X_i)\}}{P(A_i | \X_i)}, W_i  = \frac{1\{A_i = \hat{d}_n^{(-i)} (\X_i)\}}{P(A_i | \X_i)}, \bar{U}_n = n^{-1} \sum_{i=1}^n U_i$ and $\bar{W}_n = n^{-1} \sum_{i=1}^n W_i$. Since it was omitted in \cite{jiang2020}, the derivation of the influence function-inspired value function $R_i^{jk}$ is provided here. Assume $Y = O_p(1)$ and $E[W] \in (\epsilon, 1- \epsilon)$ for $0 < \epsilon < 0.5$. 
\begin{eqnarray*}
\hat{V}(d) - V_0(d) 
	&=& \frac{\sum_{i=1}^n U_i}{\sum_{i=1}^{n} W_i} - \frac{E[U]}{E[W]} \\
	&=& \frac{n^{-1}\sum_{i=1}^n (U_i- E[U])}{n^{-1}\sum_{i=1}^n W_i} - \frac{n^{-1} E[U] \cdot   \sum_{i=1}^n \l(W_i - E[W]\r)}{\l(n^{-1} \sum_{i=1}^n W_i \r)E[W]}\\
	&=& \frac{  n^{-1} \sum_{i=1}^n(U_i- E[U])}{E[W] + o_P(1)} - \frac{n^{-1}E[U] \cdot \sum_{i=1}^n (W_i - E[W])}{E[W](E[W] + o_P(1))}\\
	&=&  \frac{n^{-1} \sum_{i=1}^n(U_i - E[U])}{E[W]} - \frac{n^{-1} E[U] \cdot \sum_{i=1}^n (W_i - E[W])}{(E[W])^2} + o_P(1)
\end{eqnarray*}

According to (ii) of Theorem 18.7 in \cite{kosorok2008introduction}
\begin{equation*}
	\sqrt{n} (\hat{V} (d) - V_0 (d)) = \sqrt{n}\sum_{i=1}^n \check \psi_i + o_p(1)
\end{equation*}
for a fixed $d$, the influence function and its estimator would then be
\begin{eqnarray*}
	\check\psi_i &=& \frac{ (U_i - E[U])}{E[W]} - \frac{ E[U] (W_i- E[W])}{(E[W])^2} =  \frac{1}{E[W]} U_i - \frac{ E[U]}{(E[W])^2} W_i \\ 
	\ddot \psi_i &=& \frac{1}{\bar W} U_i- \frac{ \bar {U}}{\bar W ^2} W_i
\end{eqnarray*}
where $R_i^{jk}$ follows a similar form as $\check\psi_i$. By this definition, $\sum_{i=1}^n R^{jk}_i = 0$, which is why $R_i^{jk}$ is bias-corrected. Typically, standard deviations are scaled with $\frac{1}{n-1}$ to account for the $n-1$ degrees of freedom in the summation. Our Eq \eqref{eqn:jackknife_vhat_var} is scaled by $\frac{1}{n(n-1)}$, because we also want to adjust for correlation among the $n$ training sets as the jackknife training sets all differ by one subject. The standard error of the value estimator is $\widehat{\text{SE}} = \sqrt{\widehat{\text{Var}}(\hat{V})}$.

\subsection{Consistency of the Jackknife Estimator} \label{sec:consistency}

We study two main asymptotic properties of the jackknife estimator: consistency and asymptotic normality. Asymptotic normality is discussed in the ``Simulations'' section of \cite{jiang2020}, a summarised, updated version of which is presented in Section \ref{sec:sims}. In this section, we provide proof to the argument that ``for our case the jackknife estimates of value functions are asymptotically unbiased and their variances converge to zero as sample size increases'' \citep{jiang2020}. 

First, two assumptions need to be made.
\begin{assumption}\label{assump:pmoa1}
	$$E[P_\X(\hat{d}_n(\X) \not= \hat{d}_{n-1}(\X))] \to 0$$
\end{assumption}
\begin{assumption}\label{assump:pmoa2}
	$$E\l[ \frac{Y^2}{P(A|\X)} + \frac{1}{P(A|\X)}\r] < \infty$$
\end{assumption}
Assumption \ref{assump:pmoa1} indicates that the decision rules based on training size $n$ and training size $n-1$ are asymptotically equal in probability. This is easily satisfied when sample size goes to infinity. Assumption \ref{assump:pmoa2} requires that the expectation of the second moment of the outcome, adjusted by the propensity score, be finite. This is easily satisfied as response variables usually live in a finite range in a finite dataset. 

\begin{theorem}\label{thm:pmoa_jk_consistency}
	Given Assumptions \ref{assump:pmoa1} and \ref{assump:pmoa2}, 
	\begin{equation}\label{eqn:pmoa_consistency} \nonumber
	\frac{\sum_{i=1}^n \frac{Y_i 1\{A_i = \hat{d}_n^{(-i)} (\X_i)\}}{P(A_i |\X_i)} }{\sum_{i=1}^n \frac{1\{A_i = \hat{d}_n^{(-i)} (\X_i)\}}{P(A_i | \X_i)}} - E[Y|A = \hat{d}_n(\X)] \underset{p}{\to} 0
	\end{equation}
\end{theorem}
\begin{proof}
Let $U_i = \frac{Y_i 1\{A_i = \hat{d}_n^{(-i)} (\X_i)\}}{P(A_i | \X_i)}, W_i  = \frac{1\{A_i = \hat{d}_n^{(-i)} (\X_i)\}}{P(A_i | \X_i)}, U_n = n^{-1}\sum_{i=1}^n U_i,$ and $W_n = n^{-1} \sum_{i=1}^n W_i$. First,
\begin{equation*}
\mu_n = \E[U_n] = n^{-1} \sum_{i=1}^n \E\l[ \frac{Y_i 1\{A_i = \hat{d}_n^{(-i)} (\X_i)\}}{P(A_i | \X_i)} \r] = \E\l[ \frac{Y 1\{A = \hat{d}_{n-1}(\X)\}}{P(A|\X)} \r]
\end{equation*}
Denote $\tilde{\mu}_n = \E\l[\frac{Y 1\{A = \hat{d}_n(\X)\}}{P(A|\X)}\r]$, then 
\begin{eqnarray*}
	\mu_n - \tilde{\mu}_n &=& \E\l[ \frac{Y}{P(A|\X)} \l( 1\{A = \hat{d}_{n-1}(\X)\} - 1\{A = \hat{d}_n(\X)\}\r) \r] \\ 
	&\leq& M \E\l[1\{A = \hat{d}_{n-1}(\X)\}  - 1\{A = \hat{d}_n(\X)\}\r] + \E\l[ \frac{|Y|}{P(A|\X)} 1 \l\{ \frac{|Y|}{P(A|\X)} > M \r\}\r] \\ 
	&\to& 0
\end{eqnarray*}
where the convergence is based on Assumption \ref{assump:pmoa1} for the first term and Assumption \ref{assump:pmoa2}, which implies finite first moment, for the second term. Given the first term in Assumption \ref{assump:pmoa2}, we have the following property of the variance
\begin{eqnarray*}
	\text{Var}[U_n] &=& n^{-1} \text{Var}\l[ \sum_{i=1}^n U_i\r] \\ 
	&=& n^{-2} \sum_{i=1}^n \sum_{j=1}^n [\E(U_iU_j) - \E(U_i)\E(U_j)] \\ 
	&=& n^{-2} \sum_{i=1}^n \sum_{j=1}^n \l[ \E\l( 
	\frac{Y_i Y_j 1\{A_i = \hat{d}_n^{(-i)}(\X_i) \} 1\{A_j = \hat{d}_n^{(-j)} (\X_j)\}}{P(A_i |\X_i) P(A_j|\X_j)} \r) - \mu_n^2\r] \\ 
	&\to& n^{-2} \sum_{i=1}^n \sum_{j=1}^n \l[ \E\l( 
	\frac{Y_i Y_j 1\{A_i = \hat{d}_n^{(-i, -j)}(\X_i) \} 1\{A_j = \hat{d}_n^{(-i, -j)} (\X_j)\}}{P(A_i |\X_i) P(A_j|\X_j)} \r) - \mu_n^2\r] \\ 
	&=& n^{-2} \sum_{i=1}^n \sum_{j=1}^n \l\{\l[ \E\l( \frac{Y 1\{A = \hat{d}_{n-2}(\X)\}}{P(A|\X)}\r) \r]^2 - \mu_n^2\r\} \\ 
	&\to& n^{-2} \sum_{i=1}^n \sum_{j=1}^n (\mu_n^2 - \mu_n^2) = 0,
\end{eqnarray*}
where the convergences are based on Assumption \ref{assump:pmoa1}. Thus, we have shown that 
\begin{eqnarray*}
	\E[U_n] - \tilde{\mu}_n &\to& 0 \\ 
	\text{Var}[U_n] &\to& 0
\end{eqnarray*}
Applying the same arguments as above to $W_n$ with Assumption \ref{assump:pmoa1} and the second term in Assumption \ref{assump:pmoa2},
\begin{eqnarray*}
	\tau_n &=& \E[W_n] = \E \l[ \frac{1\{A = \hat{d}_{n-1}(\X)\}}{P(A|\X)} \r] \\ 
	\tilde{\tau}_n &=& \E\l[ \frac{1\{A = \hat{d}_n(\X)\}}{P(A|\X)}\r] = \E \l\{ \E\l[  \frac{1\{A = \hat{d}_n(\X)\}}{P(A = \hat{d}_n(\X) | \X)} \bigg| \X \r]\r\} = 1,
\end{eqnarray*}
and similarly 
\begin{eqnarray*}
	\E[W_n] - 1 &\to& 0 \\ 
	\text{Var}[W_n] &\to& 0
\end{eqnarray*}
Thus by the weak law of large numbers (WLLN), 
$$ U_n - \tilde{\mu}_n \underset{p}{\to} 0 \text{ and } W_n - 1 \underset{p}{\to} 0 $$
which yields 
$$ \frac{U_n}{W_n} - \tilde{\mu}_n \underset{p}{\to} 0 $$
by the multivariate continuous mapping theorem. This completes the proof because
$$ \tilde{\mu}_n = \E\l[ \frac{Y1\{A = \hat{d}_n(\X)\}}{P(A|\X)}\r] = \E[Y^{\hat{d}_n(\X)}] = \E[Y | A = \hat{d}_n(\X)]$$
by applying a version of Radon-Nikodym derivative (i.e. $ \frac{dP^d}{dP} = 1\{a = d(\x)\} / P(a|\x) $ where $P$ denotes the distribution of $(\X,A,Y)$ and  $P^d$ denotes the distribution of $(\X,A,Y)$ under the decision rule $d$ \citep{qian2011performance}) and since 
$$ 	\frac{\sum_{i=1}^n \frac{Y_i 1\{A_i = \hat{d}_n^{(-i)} (\X_i)\}}{P(A_i |\X_i)} }{\sum_{i=1}^n \frac{1\{A_i = \hat{d}_n^{(-i)} (\X_i)\}}{P(A_i | \X_i)}} - E[Y|A = \hat{d}_n(\X)]  = \frac{U_n}{W_n} - \tilde{\mu}_n.$$ 
\end{proof}

\section{Numerical Experiments}\label{sec:sims}
Various simulations have been done in the supplemental material of \cite{jiang2020} to evaluate the performance of the jackknife value function estimator including asymptotic normality. To generate simulated data triplets as introduced in Section \ref{sec:existing}, we assume a treatment of three categories $A$ which has a multinomial distribution with equal probability. We also assume i.i.d. uniform distribution of three variables $X_1, X_2, X_3$, the first two of which determine the decision boundary as in $E[Y] = X_1 + X_2 + \delta_0(X_1, X_2, A)$ and the last of which serves as a nuisance variable to add interference to the modeling. Four scenarios with respect to true decision boundary $\delta_0(X_1, X_2, A)$ were studied: concentric circles (scenario 1), nested steps (scenario 2), parallel diagonal lines (scenario 3), and nested parabolas (scenario 4). 
\begin{eqnarray*}\label{eqn:pmoa_simulation_scenario}
(1) \quad \delta_0(\X,A) &=& 1\{A > 0 \} (1- X_1^2 -X_2^2)(X_1^2 + X_2^2 - 3)^ {1\{A=1\}}\\  
(2) \quad \delta_0(\X,A) &=& 1\{A > 0 \} (1\{X_2 \leq \ceil{X_1 - 2 \cdot 1\{A = 2\}}\} \\& & \qquad \qquad - 1\{X_2 > \ceil{X_1 - 2 \cdot 1\{A = 2\}}\})\\
(3) \quad \delta_0(\X,A) &=& 1\{A > 0 \} (X_1 + X_2 - 1)(-X_1 - X_2 -1)^{1\{A=1\}} \\ 
(3) \quad \delta_0(\X,A) &=& 1\{A > 0 \} (X_2 - X_1^2)(X_1^2 - X_2^2 -2)^{1\{A = 1\}}
\end{eqnarray*}
These four scenarios contain both linear (steps and lines) and non-linear (circles and parabolas) relationships as well as different sizes and shapes of decision areas. Circles are the most complex scenario because there is no lines or splines that could divide a circle, and concentric circles make it even more difficult. Following the concentric circles, nested parabolas and steps are also complex and the parallel lines are relatively easier boundaries. Many sample sizes are considered $50, 100, 200, 400, 800$ and we repeat the simulation $100$ times. 

\begin{sidewaysfigure}
    \centering
    \includegraphics[scale=0.5]{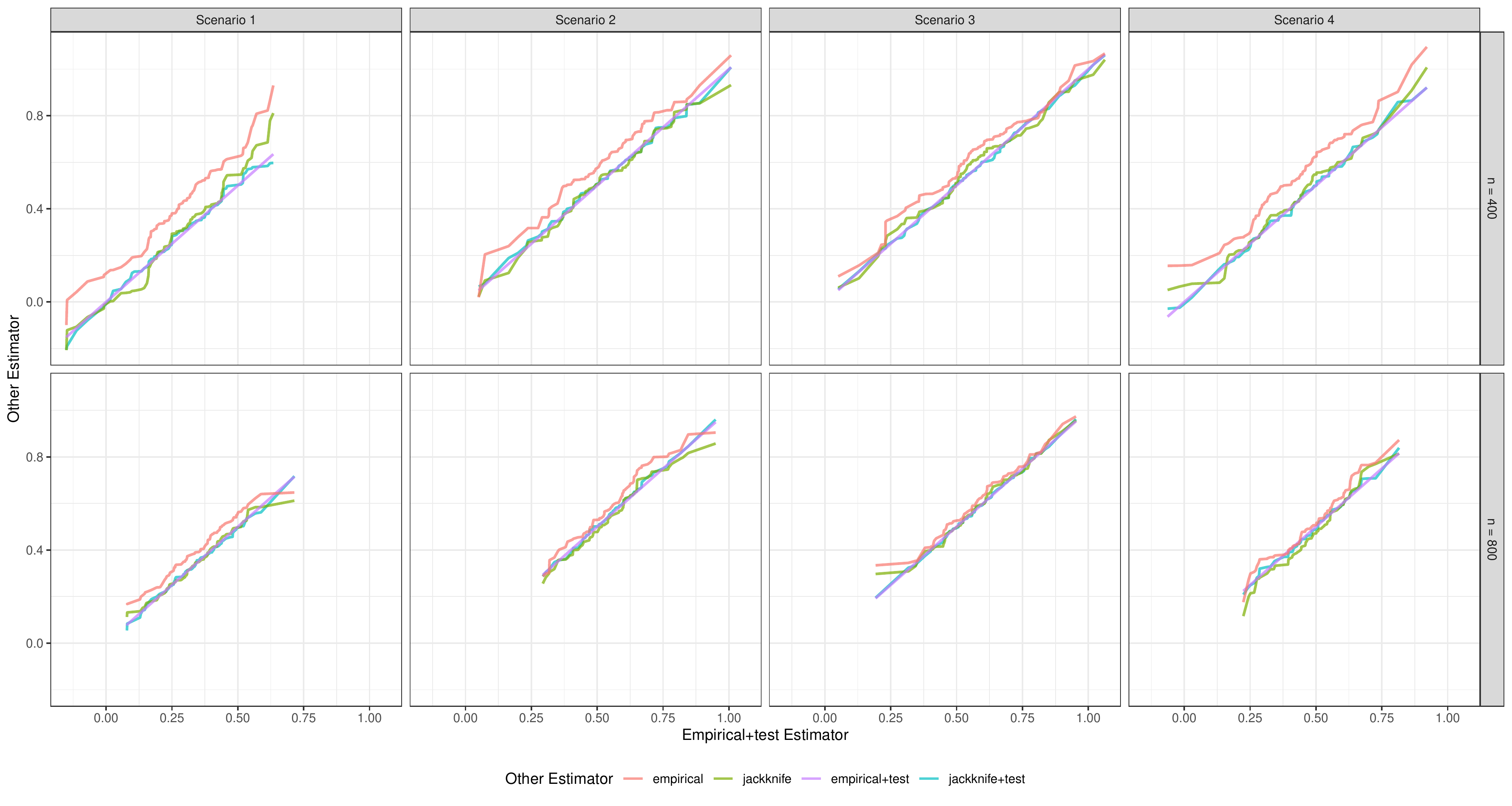}
    \caption{Q-Q plots of the distribution of estimators $\hat{V}^{jk}$ and $\hat{V}^{emp}$ versus the distribution of $\hat{V}_4$ on the KRR model across 100 simulations for $n = 400$ and $n = 800$ over 4 scenarios. (Colors: $\hat{V}^{jk}$ is green, $\hat{V}^{emp}$ is purple, $\hat{V}_1$ is red, and $\hat{V}_3$ is blue. The other two estimators $\hat{V}_1$ and $\hat{V}_3$ are explained more in \cite{jiang2020}.)}
    \label{fig:qqplot}
\end{sidewaysfigure}

We evaluate the performance of jackknife estimators from four main aspects. First, we look at how well the jackknife estimator performs similarly to an ideal estimator. More specifically, we check whether the distribution of $\hat{V}^{jk}$ (also known as $\hat{V}_2$ in the simulation results of \cite{jiang2020}) is identical to the distribution of the ideal empirical estimator $$V^{emp} =  \frac{\sum_{i=1}^n \tilde{Y}_i 1\{\tilde{A}_i =\hat{d}_n(\tilde{\X}_i)\}/ P(\tilde{A}_i | \tilde{\X}_i)}{\sum_{i=1}^n 1\{\tilde{A}_i =\hat{d}_n(\tilde{\X}_i)\}/ P(\tilde{A}_i | \tilde{\X}_i)}$$  (denoted as $\hat{V}_4$ in the simulation results of \cite{jiang2020}). The entire dataset of size $n$ is used to train $d$ and then tested on an independent dataset of the same distribution and size as the training set. This is the most ideal estimator because training and test sets are completely separate and of the same size but unattainable in reality. Simulation results (Supplemental Figure 4 in \cite{jiang2020}) showed that the two estimators have similar distributions especially when $n$ is larger and decision boundaries are simpler. Using the same kernel ridge regression (KRR) model, we increase the sample size to $800$ and plot the distribution comparison of the two estimators in Figure \ref{fig:qqplot}. We have compared $n=50$ and $n=400$ before so we are only comparing $n=400$ and $n=800$ here for cleaner plots. It is easy conclude from Figure \ref{fig:qqplot} that our jackknife estimator can perform as well as the ideal estimator. With a higher sample size, the distribution is more gathered towards the center and the distributions of all estimators are closer to $\hat{V}^{emp}$ (the straight purple line). 

Second, we look at the variability of our jackknife estimator by measuring how much the confidence interval (CI) of $\hat{V}^{jk}$ covers the true value function $V_0$. In \cite{jiang2020}, a $95\%$ CI is calculated for each simulation with the formula $\hat{V}^{jk} \pm z_{0.975} \cdot SE(\hat{V}^{jk})$ where $z_{0.975}$ is the standard normal quantile of $97.5\%$. We follow the same definition of coverage, which is defined as the proportion of $100$ simulations whose $95\%$ CI contains the truth $V_0$. Using the same KRR model, we summarize the coverages in Table \ref{tab:coverage}, adding a larger sample size $n=800$ to Supplemental Table 3 in \cite{jiang2020}. There are a few fluctuations across sample sizes for scenarios 1 and 2 given the $5\%$ Monte Carlo error (the maximum standard error of the coverage), but we see a general trend where the coverage is closer to $95\%$ as sample size grows. 

\begin{table}
	\centering 
	\caption{Coverage of the empirically true estimator $V_0$ with 95\% CI of $\hat{V}^{jk}$}
	\begin{tabular}{lrrrrr}\hline 
		Sample Size & 50 & 100 & 200 & 400 & 800 \\ \hline 
		Scenario 1 & 84\% & 87\% & 92\% & 88\% & 96\% \\ 
		Scenario 2 & 91\% & 91\% & 96\% & 96\% & 92\% \\
		Scenario 3 & 93\% & 97\% & 96\% & 96\% & 96\% \\
		Scenario 4 & 89\% & 88\% & 90\% &94\% & 93\% \\ \hline 
	\end{tabular}
	\label{tab:coverage}
\end{table}

Next, we check the power of the test statistic used to compare the optimal precision medicine model (PMM) and the optimal zero-order model (ZOM). ZOMs are single, fixed treatment rules that assign the same treatment to all subjects, which is considered simple but not necessarily works for each subject's situation. The model comparison is carried by a two-sample Z-test with test statistic 
$$T^{sim}(\hat{d}_{\text{PMM}}, \hat{d}_{\text{ZOM}}) = \frac{\widehat{V}(\hat{d}_{\text{PMM}}) - \widehat{V}(\hat{d}_{\text{ZOM}})}{\sqrt{\frac{\sum_{i=1}^n (R_{\text{PMM,i}} - R_{\text{ZOM, i}})^2 }{n(n-1)}}}$$ \citep{jiang2020}. Power is the true positive rate of correctly detecting a significant effect of PMM over ZOM and is estimated as the proportion of $100$ simulations whose p-values are under $0.05$. Results (Supplemental Table 4 in \cite{jiang2020}) showed low powers when sample size was low ($n \leq 200$) but the estimated power increased by a lot from $n=200$ to $n=400$. 
\begin{table}
	\centering 
	\caption{Estimated power of jackknife $T^{sim}$ based on 100 simulations}
	\begin{tabular}{lrrrrr}\hline 
		Sample Size & 50 & 100 & 200 & 400 & 800 \\ \hline 
		Scenario 1 & 13\% & 7\% & 18\% & 38\% & 67\%\\ 
		Scenario 2 & 16\% & 13\% & 23\% & 34\% & 49\%\\
		Scenario 3 & 15\% & 24\% & 41\% & 81\% & 94\%\\
		Scenario 4 & 11\% & 15\% & 35\% & 56\% & 85\%\\ \hline 
	\end{tabular}
	\label{tab:power}
\end{table}
Table \ref{tab:power} contains the new sample size $n=800$ and we see a bigger increase for all scenarios as we increase the sample size beyond $n=400$. Simpler scenarios such as lines and parabolas have higher power than more complex scenarios such as circles and steps. There is still room to improve power for scenarios 1 and 2.

Last, we use simulations to study the asymptotic property of the jackknife estimator. \cite{jiang2020} used the following shifted test statistic to measure how far apart are the PMM and ZPM value functions. 
\begin{equation}\label{eqn:pmoa_asymp_norm}
T^{sim}_0 = \frac{[\hat{V}(\hat{d}_{\text{PMM}}) - \hat{V}(\hat{d}_{\text{ZOM}})] - [V_0(\hat{d}_{\text{PMM}}) - V_0(\hat{d}_{\text{ZOM}}) ]}{\sqrt{\frac{\sum_{i=1}^n (R_{\text{PMM},i} - R_{\text{ZOM},i})^2}{n(n-1)}}}
\end{equation}
Both the Shapiro-Wilk test (Supplemental Table 5 in \cite{jiang2020}) and Q-Q plots of the distribution of $T_0^{sim}$ over $100$ simulations versus standard normal distribution (Supplemental Figure 5 in \cite{jiang2020}) are used to learn the normality of $T_0^{sim}$. Table \ref{tab:shapirowilk} extends the Shapiro-Wilk test to sample size of $n = 800$. 
\begin{table}
	\centering 
	\caption{P-values of Shapiro-Wilk test of normality on jackknife $T_0^{sim}$}
	\begin{tabular}{lrrrrr}\hline 
		Sample Size & 50 & 100 & 200 & 400 & 800 \\ \hline 
		Scenario 1 & 0.20 & 0.37 & 0.15 & 0.18 & 0.74\\ 
		Scenario 2 &0.85 & 0.92 & 0.41 & 0.79 & 0.64 \\
		Scenario 3 & $<$0.01 & 0.99 & 0.13 & 0.61 & 0.87 \\
		Scenario 4 & $<$0.01 & 0.67 & 0.81 & 0.84 & 0.86 \\ \hline 
	\end{tabular}
	\label{tab:shapirowilk}
\end{table}
The p-values are still well above $0.05$, indicating that there is not enough evidence to reject the null hypothesis that $T_0^{sim}$ has a standard normal distribution for all scenarios and sample sizes $n \leq 100$. Next we look at a Q-Q plot (Figure \ref{fig:normality}) to visually check if the distribution of $T_0^{sim}$ is normally distributed. Similar as before, we only show $n=400$ and $n=800$ here to focus on higher sample sizes because \cite{jiang2020} compared lower sample sizes $n=50$ and $n=400$. As sample size doubled to $n=800$, the scatter points of the distribution of $T_0^{sim}$ live more on the straight diagonal line, implying closer to normality than $n=800$ for all scenarios. Albeit there are still some outliers (especially for scenario 2), we see improvement at the two tails for other scenarios.
\begin{sidewaysfigure}
    \centering
    \includegraphics[scale=0.6]{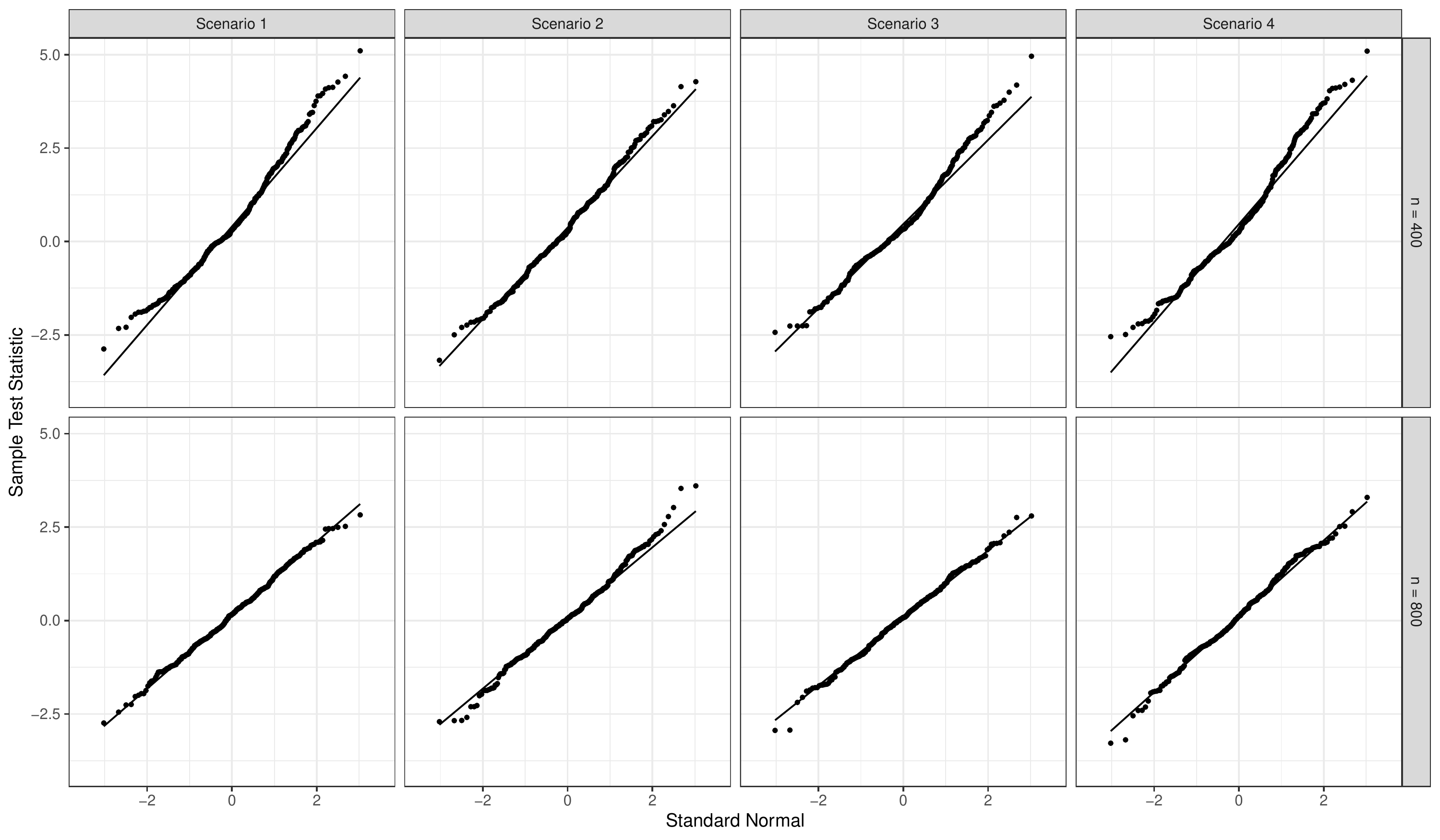}
    \caption{Q-Q plots of the distribution of the jackknife test statistic $T_0^{sim}$ across $100$ simulations versus the standard normal distribution}
    \label{fig:normality}
\end{sidewaysfigure}

\section{Summary}
\cite{jiang2020} used the jackknife method to estimate value functions and found the optimal individualized treatments for participants enrolled in a knee osteoarthritis clinical trial.  We expand the discussion and justification with additional statistical background in this technical report. We have introduced precision medicine and existing methods in greater detail (Sections \ref{sec:intro} and \ref{sec:existing}), provided the definition and derivation of jackknife estimators (Section \ref{sec:definition}), showed theoretical evidence for the consistency of jackknife estimators (Section \ref{sec:consistency}), and numerical evidence for the performance and properties of jackknife estimators with a higher sample size (Section \ref{sec:sims}). We recommend reading \cite{jiang2020} for clinical application and interpretation of the optimal ITR of knee osteoarthritis as well as an overall understanding of the pipeline, and recommend reading this article for more in-depth statistical derivation and methodology.

\bibliographystyle{apalike}
\bibliography{allbib}
\end{document}